%% file: main.tex
\documentclass[11pt]{article}
\usepackage{microtype}
\usepackage{graphicx}
\usepackage{booktabs} %
\usepackage{fullpage}
\usepackage{hyperref}
\usepackage{algorithm, algorithmic}
\usepackage{natbib}
\usepackage{xcolor}

\usepackage{caption}
\usepackage{subcaption}
\usepackage{amsmath, amsfonts, amsthm}
\newtheorem{definition}{Definition}
\newtheorem{theorem}{Theorem}

\newtheorem{lemma}{Lemma}
\newcommand{\BB}{\mathbb{B}}
\newcommand{\EE}{\mathbb{E}}
\newcommand{\RR}{\mathbb{R}}
\newcommand{\cE}{\mathcal{E}}
\newcommand{\cX}{\mathcal{X}}
\newcommand{\cY}{\mathcal{Y}}
\newcommand{\cO}{\mathcal{O}}
\newcommand{\cQ}{\mathcal{Q}}
\newcommand{\sigmamd}{\sigma_{\text{md}}}
\newcommand{\sigmadmd}{ \sigma^d_{\text{md}} }
\newcommand{\mean}{{\bar{x}}}
\newcommand{\seed}{s}
\newcommand{\smooth}{H}
\newcommand{\est}{\hat{x}}
\usepackage{thmtools}
\usepackage{thm-restate}

\newcommand{\monotone}{\emph{$k$-interval} }
\newcommand{\diam}{D}
\newcommand{\sg}{{g}}

\newcommand{\zs}[1]{\textcolor{black}{#1}}
\renewcommand{\th}[1]{\textcolor{black}{#1}}
\newcommand{\cq}{correlated quantization }
\newcommand{\Cq}{Correlated quantization }

\newcommand{\ed}{\stackrel{\mathrm{def}}{=}}

\title{Correlated quantization for distributed mean estimation and optimization}

\author{Ananda Theertha Suresh \and Ziteng Sun \and Jae Hun Ro \and Felix Yu}
\date{
\texttt{\{theertha, zitengsun, jaero, felixyu\}@google.com} \\[2ex]
Google Research, New York}

\begin{document}

\maketitle

\begin{abstract}
We study the problem of distributed mean estimation and optimization 
under communication constraints. We propose a \cq protocol whose 
leading term in the error guarantee depends on the mean deviation of 
data points rather than only their absolute range. The design doesn't need 
any prior knowledge on the concentration property of the dataset, which is 
required to get such dependence in previous works.  We show that applying 
the proposed protocol as a sub-routine in distributed optimization 
algorithms leads to better convergence rates.  We also prove the optimality 
of our protocol under mild assumptions. Experimental results show that 
our proposed algorithm outperforms existing mean estimation protocols on 
a diverse set of tasks.
\end{abstract}

\section{Introduction}

Large-scale machine learning systems often require the distribution of data across multiple devices.
For example, in federated learning \citep{kairouz2019advances}, data is distributed across user devices such as cell phones, and the machine learning models are trained with adaptive stochastic gradient descent methods or their variations like federated averaging \citep{mcmahan2017communication}. Such algorithms require multiple rounds of communication between the devices and the centralized server. At each round, the devices send model updates to the server, and the server aggregates the data and outputs a new model.

In many scenarios like federated learning, the data from devices are sent to the server over wireless channels.  Communication between devices and the server, especially the uplink communication, is a bottleneck. 
This has resulted in a series of works on compression and quantization methods to reduce the communication cost \citep{konevcny2016federated, lin2017deep, alistarh2017qsgd}.

At the heart of these algorithms is the distributed mean estimation protocol 
where each client has a model update (in the form of a vector). Each client 
compresses its update and transmits the compressed version to the server. 
The server then decompresses and aggregates the updates to approximate 
the mean of the updates.  In this work, we study the problem of distributed 
mean estimation and provide the first algorithm whose leading term of the 
error depends on the mean deviation of the inputs rather than only their 
absolute 
range without additional side information. We then use these results to 
provide improved convergence guarantees for distributed optimization 
protocols.
Before we proceed further, we need a few definitions. 

\paragraph{Distributed mean estimation.} Let $\cX \subset \mathbb{R}^d$ be the input space and  $x^n = x_1, x_2, \ldots, x_n$ be the $n$ data points where each $x_i \in \cX$. For most results, we assume $\cX = \mathbb{B}^d(R)  \ed \{ x \in \mathbb{R}^d \mid \| x\|_2 \le R \}$, the $\ell_2$ ball of radius $R$. We denote the mean of the vectors as
\[
\mean = \frac{1}{n} \sum^n_{i=1} x_i.
\]
In compression, these $x_i$s are encoded at the clients and then decoded at the server \citep{mayekar2020ratq}. $Q_i$, the quantizer (encoder) at client $i$ is a (possibly randomized) mapping from $\cX \to \cY$, where $\cY$ is the quantized space. With the slight abuse of notation let $Q^n = Q_1, Q_2, \ldots Q_n$, be the set of quantizers.  Each client encodes $x_i$ and sends $Q_i(x_i)$ to the server. The server then decodes $Q^n(x^n)$ to get an estimate of the mean $\est$.
Following earlier works \citep{suresh2017distributed, mayekar2020ratq}, we measure the performance of the quantizer in terms of mean squared error,
\[
\cE(Q^n, x^n) \ed \EE  \| \est - \mean \|_2 ^2,
\]
where the expectation is over the public and private randomness of the algorithm.

\paragraph{Distributed optimization.} We consider solving the following optimization task using distributed stochastic gradient descent (SGD) methods. Let $F: \mathbb{R}^d \rightarrow \mathbb{R}$ be an objective function and the goal is to minimize the $F$ over an $\ell_2$-bounded space $\Omega \ed \{w \in \mathbb{R}^d \mid \|w\|_2 \le \diam \}.$ Motivated by the federated learning setting, we assume there are $T$ rounds of communication. At round $t$, a set of $n$ clients is involved, and each of them has access to a stochastic gradient oracle of $F$, denoted by $\sg$. 

In distributed SGD, after selecting a random initialization $w_0 \in \Omega$, at round $t$, each client queries the oracle at $w_t$. Under communication constraints, users must quantize their obtained gradient with limited bits and send it to the server. The server then uses these quantized messages to estimate the true gradient of $F(w)$, which is the average of local gradients. We denote this estimate as $\widehat{\nabla}F(w_t)$. The server then updates the parameter with some learning rate
\[
    w_{t+1} = w_t - \eta_t \widehat{\nabla}F(w_t),
\]
and projects it back to $\Omega$, which is sent to all clients in the next round.

\zs{Under smoothness assumptions on $F$,} standard results in optimization, e.g., \citet{bubeck2015convex}, allow us to obtain convergence results given mean squared error guarantees for the mean estimation primitive $ \EE [ (\widehat{\nabla}F(w_t) -\nabla F(w_t))^2 ]$. Hence we will focus on analyzing error guarantees on the mean estimation task and discuss their implications on distributed optimization.

\section{Related Works}

The goal of distributed mean estimation is to estimate the empirical mean without making distributional assumptions of the data. This is different from works estimating the mean of the underlying distributional model \citep{zhang2013information, garg2014communication, braverman2016communication, cai2020distributed, acharya2021distributed}. To achieve guarantees in terms of the deviation of the data, these techniques rely on the distributional assumption, which is not applicable in our setting.

The classic algorithm for this problem is stochastic scalar quantization, where each dimension of the data is stochastically quantized to one of the fixed values (such as  0 or 1 in stochastic binary quantization). This provides an unbiased estimation with reduced communication cost. It has been shown that adding random rotation reduces quantization error \citep{suresh2017distributed} and variable-length coding provides the near-optimal communication-error trade-off \citep{alistarh2017qsgd, suresh2017distributed}.
Many variants and improvements of the scalar quantization algorithms exist. For example, Terngrad \citep{wen2017terngrad}  and 3LC \citep {lim20193lc} use a three-level stochastic quantization strategy. SignSGD uses the sign of the coordinate of gradients rather than quantizing it \citep{bernstein2018signsgd}. 1-bit SGD uses error-feedback as a mechanism to reduce the error in quantization \citep{seide20141}; error-feedback \citep{stich2020error} is orthogonal to our work and can be potentially used in combination. \citet{mitchell2022optimizing} proposes to learn the quantizer leveraging data distribution across the clients using rate-distortion theory. \citet{vargaftik2021drive} proposes DRIVE, an improvement of the random rotation method by replacing the stochastic binary quantization with the sign operator.  This method is shown to outperform other variants of scalar quantization. Recent work of \citet{Vargaftik2022eden} generalizes DRIVE to handle any communication budget. \th{Other techniques include Sparsified SGD \citep{stich2018sparsified} and sketching based approaches \citep{ivkin2019communication}.}

Beyond scalar quantization,  vector quantization may lead to higher worst-case communication cost \citep{gandikota2021vqsgd}.  Kashin's representation has been used to quantize a $d$-dimensional vector using less than $d$ bits \citep{caldas2018expanding, chen2020breaking, safaryan2020uncertainty}.  \citet{davies2020distributed} use the lattice quantization method which will be discussed below. 

More broadly, our work is also related to non-quantization methods to improve the communication cost of distributed mean estimation, often under the context of distributed optimization. Examples include gradient sparsification \citep{aji2017sparse, lin2017deep, basu2019qsparse} and low rank decomposition \citep{wang2018atomo, vogels2019powersgd}. These methods require  assumptions of the data such as high sparsity or low rank. \zs{The idea of using correlation between local compressors has also been considered in \citet{szlendak2021permutation} for gradient sparsification. The paper uses shared randomness to select coordinates without-replacement across clients, which is also shown to be advantageous to independent masking. Without compression, \cite{yun2022minibatch} studies the effect of randomness in local data reshuffling. The paper shows that for both local and mini-batch SGD, sampling indices without-replacement can outperform sampling with-replacement counterparts in certain regimes.}

Perhaps closest to our work is that of \citet{davies2020distributed, mayekar2021wyner}, who proposed algorithms with error that depends on the variance of the inputs. However, these works all need certain side information about the inputs, and deviate from our work in two ways: first in \citet{davies2020distributed}, the clients need to know the input variance. Secondly, both \citet{davies2020distributed, mayekar2021wyner} require the server to know one of the client values to a high accuracy. Finally, their information theoretically optimal algorithm is not computationally efficient, and their efficient algorithm is sub-optimal in logarithmic factors. \th{We note that correlated random variables are also used in  \citet{mayekar2021wyner}. They use the same random variable to quantize both the input vector and the side information. This is very different from our proposed approach,
which generates correlated random variables 
which are used for different input vectors.}
\section{Our Contributions}

We propose  \cq  that only requires a simple modification in the standard stochastic quantization algorithm. \Cq  uses shared randomness to introduce correlation between local quantizers at each client, which results in improved error bounds. In the absence of shared randomness, it can be simulated by the server sending a seed to generate randomness to all the clients. We first state the error guarantees below.

In one dimension, if all values lie in the range $[l, r]$, the error of standard stochastic quantization with $k$ levels scales as \citep{suresh2017distributed}
\begin{equation*}
\cO \left( \frac{(r - l)^2}{nk^2} \right).
\end{equation*}
In this work, we show that the modified algorithm (Algorithm~\ref{fig:one_d_k_bits}) has error that scales as
\begin{equation*}
\cO \left( \min \left(\frac{\sigmamd (r - l)}{nk},\frac{(r - l)^2}{nk^2} \right) + \frac{(r - l)^2}{n^2k^2} \right),
\end{equation*}
where $\sigmamd$ is the empirical mean absolute deviation of points defined below:
\begin{equation}
    \label{eq:sigmamd}
    \sigmamd \ed \frac{1}{n} \sum_{i = 1}^n \left| x_i - \mean \right|.
\end{equation}
Informally, $\sigmamd$ models how concentrated the data points are. Compared to other commonly used concentration measures, such as worst-case deviation  \citep{mayekar2021wyner}  
\[
\sigma_{\text{max}} \ed \max_i |x_i - \mean|,
\]
and standard deviation \citep{davies2020distributed}
\[
\sigma \ed \sqrt{\frac{1}{n} \sum_{i = 1}^n \left( x_i - \mean \right)^2},
\]
it holds that $\sigmamd \le \sigma \le \sigma_{\text{max}} \leq r - l$. Hence our result implies bounds in terms of these concentration measures as well. 

When $\sigmamd < \frac{r - l}{k}$, i.e., the data points are ``close" to each other, the proposed algorithm has smaller error than stochastic quantization. 
Notice that it was shown in previous works that better error guarantees can be obtained when the data points have better concentration properties. However these works rely on knowing a bound on the concentration radius \citep{davies2020distributed} or side information such as a crude bound on the location of the mean \citep{davies2020distributed, mayekar2021wyner}. 
Different from the above, our proposed scheme doesn't require any side information.
Moreover, we remark that our proposed scheme only requires a simple modification of how randomness is generated in the standard stochastic quantization algorithm while these algorithms are based on sophisticated encoding schemes based on the availability of prior information.

\paragraph{Lower bound.} \th{When $\sigmamd$ is small,} we further show that in the one-dimensional case, our obtained bound is optimal for any \monotone quantizers (see Definition~\ref{def:monotone}), which is commonly used in many state-of-the art compression algorithms in distributed optimization including stochastic quantization and our proposed algorithm. Moreover, when each client is only allowed to use one bit (or constant bits), the obtained bound is information-theoretically optimal for any quantizers.

\paragraph{Extension to higher dimensions.} In high dimensions, if all values lie in $\mathbb{B}^d(R)$, the error of the min-max optimal algorithms with $k$ levels of quantization scales as \citep{suresh2017distributed}
\begin{equation*}
     \tilde{\cO} \left( \frac{R^2}{nk^2} \right).
\end{equation*}
We show that an improvement similar to the previous result: 
\begin{equation*}
  \tilde{\cO} \left( \min \left(\frac{\sigmadmd R}{nk},\frac{R^2}{nk^2} \right) + \frac{R^2}{n^2k^2} \right), 
\end{equation*}
where $\sigmadmd$ is the average $\ell_2$ distance between data points, 
\begin{equation}
    \label{eq:sigmadmd}
\sigmadmd = \frac{1}{n} \sum^n_{i=1} \|x_{i} - \mean\|_2.
\end{equation}
Similar to the one-dimensional case, it can be shown that $\sigmadmd$ is upper bounded by $\max_i \| x_i - \mean\|_2$ and $(\frac1n\sum_{i = 1}^n \|x_i - \mean\|_2^2)^{1/2}$. And hence the same bound holds for these measures as well.

We note that our dependence on standard deviation is linear. This coincides with the results of \citet{mayekar2021wyner} for mean estimation with side information. While the results of one are not directly applicable to the other, understanding the two results in the high-dimensional case remains an interesting direction to explore.

\paragraph{Distributed optimization.}
Turning to the task of distributed optimization, the proposed mean estimation algorithm can be used as a primitive in distributed SGD algorithm. Following stochastic optimization literature, we assume the gradient oracle $g$ satisfies $\forall w \in \Omega$,
\begin{itemize}
    \item Unbiasedness: $\EE[\sg(w)] = \nabla F(w)$. 
    \item Lipschitzness: $\|\sg(w)\|_2 \le R$.
    \item Bounded variance: $\EE[\|\sg(w) - \nabla F(w))\|_2^2] \le \sigma^2$.
\end{itemize}
We also assume that the function is convex and smooth. A function $F$ is $H$-smooth, if for all $w, w' \in \Omega$,
\[
\| \nabla F(w) - \nabla F(w') \|_2 \leq H \|w - w'\|_2.
\]
Using standard results on smooth convex optimization (e.g., Theorem 6.3 in \citet{bubeck2015convex}) and estimation error guarantee in Corollary~\ref{cor:entropycq}, we obtain the following bound. 
\newcommand{\err}{\triangle}
\begin{theorem}
    Suppose the objective function $F(w)$ is convex and $\smooth$-smooth,
    using \cq with $k = \sqrt{d}$ levels as the mean estimation primitive, distributed SGD with $T$ rounds and learning rate $\frac{1}{H + 1/\eta}$ yields
        \[
       \EE \left[ F\left(\frac{1}{T} \sum_{t = 1}^T w_t\right) \right] - \min_{w \in \Omega} F(w) \le \left(\smooth + \frac1\eta \right) \frac{ D^2}{T} + \frac{\eta \err^2}{2},
    \]
    where $\eta > 0$ and 
    $\err^2    = \cO \left( 
\frac{\sigma R}{n} + \frac{R^2}{n^2} \right)$\footnote{Note that here we ignore the statistical variance term $\frac{\sigma^2}{n}$ since it is always smaller than the quantization error. The same applies for the classic stochastic rounding algorithm.}.
    Moreover, each client only needs to send $\Theta(d)$ bits (constant bits per dimension) in each round. 
\end{theorem}
Optimization algorithms based on stochastic rounding \citep{alistarh2017qsgd, suresh2017distributed} has convergence rate of the same form with $\err'^2 = \cO \left(
\frac{R^2}{n} \right)$ under the same communication complexity.
Notice that the mean squared error (and therefore the convergence rate) we obtain is always better than that of the classic stochastic quantization algorithm, and when $\sigma < R$ the new algorithm gives faster convergence. 
This corresponds to the case where the clients' local gradients are better concentrated than their absolute range, which is a mild assumption.

\section{A Toy Example}
\begin{figure}
\begin{center}
\includegraphics[scale=0.4]{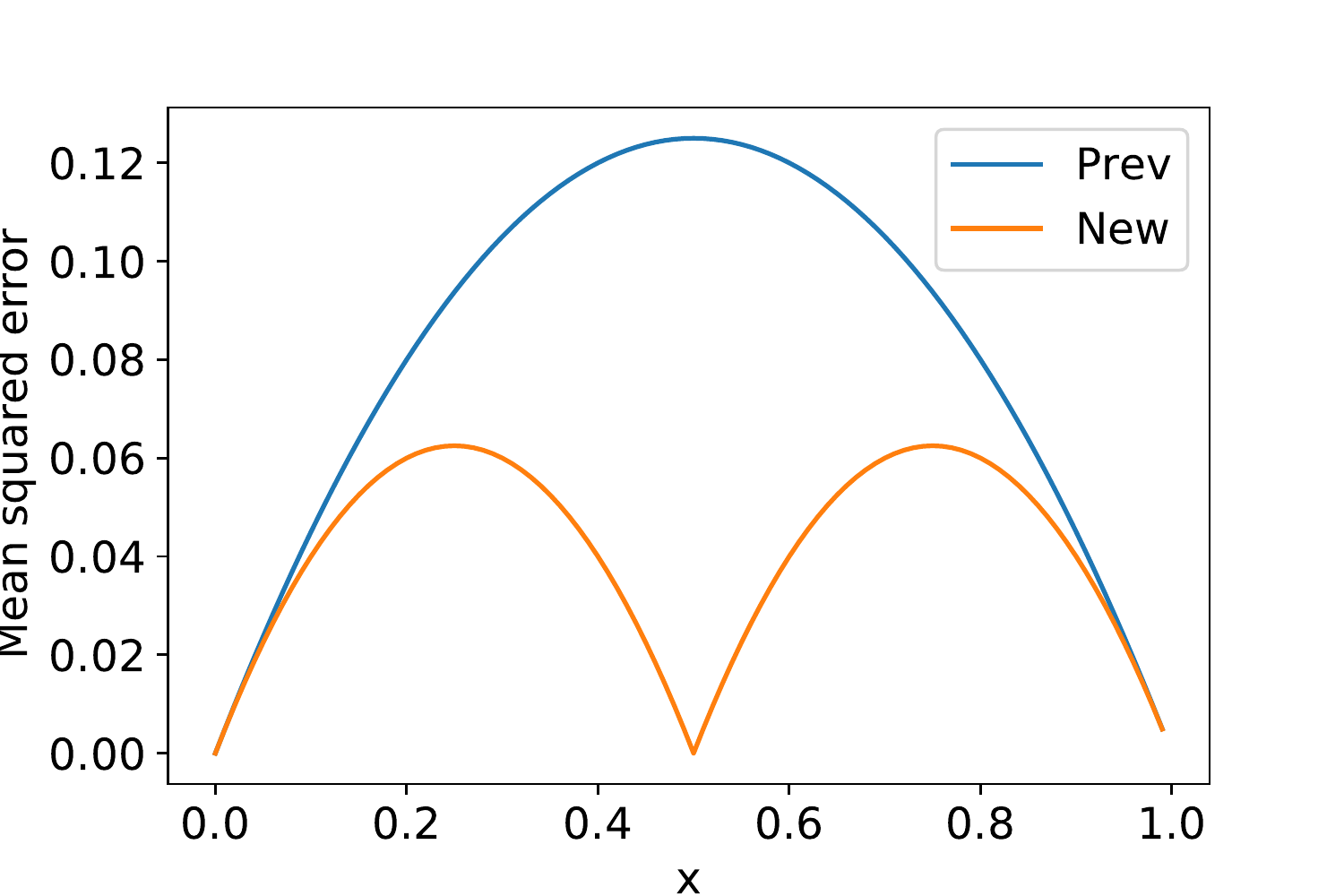}
\vspace{-10pt}
\end{center}
\caption{Mean squared error for the toy problem.}
\vspace{-10pt}
\label{fig:toy}
\end{figure}
Before we proceed further, we motivate our result with a simple example. 
Suppose there are $n$ devices and device $i$ has $x_i \in [0, 1]$.  
Further, let's consider the simple scenario in which each client can only send one bit i.e., $Q_i$ can take only two values. The popular algorithm for mean estimation is stochastic quantization, in which client $i$ sends $Q_i(x_i)$ given by
\begin{equation}
    \label{eq:sround}
Q_i(x_i) 
= \begin{cases}
1 \text{ with probability } x_i, \\
0 \text{ with probability } 1 - x_i.
\end{cases}
\end{equation}
Notice that $Q_i(x_i)$ takes only two values and the server computes the mean by computing their average
\[
\est = \frac{1}{n} \sum^n_{i=1} Q_i(x_i).
\]
We first note that $Q_i(x_i)$ is unbiased
\[
\EE[\est] = \mean.
\]
The mean squared error of this protocol is
\[
\cE(Q^n, x^n) = \frac{1}{n^2} \sum^n_{i=1} x_i (1-x_i).
\]
To motivate our approach, we consider the special case when $n=2$ and further assume $x_1 = x_2 = x$. In this case, the mean squared error of the stochastic quantizer is 
\begin{equation}
\label{eq:srmse}
\cE(Q^2, x^2) = \frac{x(1-x)}{2}.
\end{equation}
The key insight of \cq is that if the first client rounds up its value,  the second client should round down its value to reduce the mean squared error. To see this, we first rewrite the original stochastic quantization $Q_i$ slightly differently. For each $i$, let $U_i$ be an independent uniform random variable in the range $[0, 1]$. Then we can rewrite $Q_i$ as
\begin{align*}
Q_i(x_i) &= 1_{U_i < x_i}.
\end{align*}
To see the equivalence of the above definition and~\eqref{eq:sround}, observe that the probability of $U_i < x_i$ is $x_i$. Hence $1_{U_i < x_i}$ is one with probability $x_i$ and zero otherwise.

For the special case of $n=2$, we propose to modify the estimator by making $U_i$s dependent on each other. In particular, let $U_1$ be a uniform random variable in the range $[0, 1]$ and let $U_2 = 1 - U_1$. This can be implemented using shared randomness. Let the modified estimator be as before
\begin{align*}
Q'_i(x_i) &= 1_{U_i < x_i}.
\end{align*}
Since $Q_i$ is an unbiased estimator, $Q'_i$ is also an unbiased estimator. We now compute the mean squared error of $Q'$
\begin{align*}
  \cE(Q'^2, x^2)  
   &= \EE \left[ \left( \frac{Q'_1(x_1) + Q'_2(x_2)}{2} \right)^2 \right] - x^2 \\
    & = \frac{1}{4} \EE \left[ \left( 1_{U_1 < x} +  1_{1 -U_1 < x} \right)^2 \right] - x^2 \\
    & = \frac{1}{4} \EE \left[ 1_{U_1 < x} +  1_{1 -U_1 < x} + 2 \cdot 1_{U_1 < x} 1_{1-U_1 < x}  \right] - x^2 \\
  &= \frac{1}{4} \EE \left[ 1_{U_1 < x} +  1_{1 -U_1 < x} + 2 \cdot 1_{1 - x < U_1 < x} \right] - x^2 \\
    & = \frac{1}{4} \left( x + x + 2\max(2x-1, 0)  \right) - x^2 \\
    & = \frac{x}{2} + \max \left(x-\frac{1}{2}, 0 \right) - x^2.
\end{align*}
We plot the mean squared error of the original quantizer $Q$ and the new quantizer $Q'$ in Figure~\ref{fig:toy}.
Observe that the above mean squared error is uniformly upper bounded by the mean squared error of the original quantizer~\eqref{eq:srmse}.   Our goal is to propose a similar estimator for $n$ devices, in $d$ dimensions, that has uniformly low error compared to $Q$ even when the devices have different values of $x_i$. 

\section{Correlated Quantization in One Dimension}

We first extend the above idea and provide a new quantizer for bounded scalar values. For simplicity, we first assume each $\cX = [0,1]$. Recall that the goal is to estimate 
$
\frac{1}{n} \sum^n_{i=1} x_i.
$
Let $Q_i(x_i)$ be some quantization of $x_i$. We approximate the average by 
$
\frac{1}{n} \sum^n_{i=1} Q_i(x_i).
$
As we discussed in the previous section, we propose to use 
\begin{align*}
Q_i(x_i) &= 1_{U_i < x_i},
\end{align*}
where $U_i$s are uniform random variables between zero and one, but are now correlated across the clients. We generate $U_i$s as follows.  Let $\pi$ denote a random permutation of $\{0, 1, 2,\ldots, n-1\}$. Let $\pi_i$ denote the $i^\text{th}$ value of this permutation. Let $\gamma_i$ is a uniform random variable between $[0, 1/n)$. With these definitions, we let
\[
U_i = \frac{\pi_i}{n} + \gamma_i.
\]
Observe that for each $i$, $U_i$ is a uniform random variable over  $[0, 1]$. However, they are now correlated across clients.
Hence the quantizer can be written as
\[
 Q_i(x_i) = 1_{\frac{\pi_i}{n} + \gamma_i < x_i}.
 \]
We illustrate why this quantizer is better with an example. If all clients have the same value $s/n, s \in \{0, 1, \ldots, n\}$, the new quantizer can be written as
\begin{align*}
    \frac{1}{n} \sum^n_{i=1}  Q_i(s/n) 
   &  =  \frac{1}{n} \sum^n_{i=1}  1_{\frac{\pi_i}{n} + \gamma_i < \frac{s}{n}} 
     \stackrel{(a)}{=}  \frac{1}{n} \sum^n_{i=1}  1_{\frac{\pi_i}{n} < \frac{s}{n}} \\
     & \stackrel{(b)}{=}  \frac{1}{n} \sum^n_{j=1}  1_{\frac{j-1}{n} < \frac{s}{n}} 
    =  \frac{1}{n} \sum^n_{j=1}  1_{j \leq s} 
     = \frac{s}{n},
\end{align*}
where $(a)$ uses the fact that  the value of $\gamma_i$ does not change $ 1_{\frac{\pi_i}{n} + \gamma_i < \frac{s}{n}} $ for this example and $(b)$ uses the fact that $\pi$ is a random permutation of $\{0, 1, 2,\ldots, n-1\}$. We have shown that the correlated quantizer has zero error in the above example. 
In contrast, the standard stochastic quantizer outputs
\[
\frac{1}{n} \sum^n_{i=1}  1_{U_i < \frac{s}{n}} \stackrel{(c)}{=} \frac{1}{n} \sum^n_{i=1}  1_{\frac{\lfloor U_i  n \rfloor}{n} < \frac{s}{n}} = \frac{1}{n} \sum^n_{i=1}  1_{\frac{\pi'_i }{n} < \frac{s}{n}},
\]
where $(c)$ follows from the fact that $s$ is an integer and $\pi'_i = \lfloor U_i  n \rfloor$, a uniform random variable from the set $\{0, 1, \ldots, n-1\}$. Moreover, $\pi_i'$s are independent.

The above example also provides an alternative view of the proposed quantizer. If  $\forall i$, $x_i = s/n$, then the random variables in the standard stochastic random quantizer can be viewed as \emph{sampling-with-replacement} from the set $\{0/n, 1/n\ldots, (n-1)/n\}$ while the random variables in the correlated quantizer can be viewed as an  \emph{sampling-without-replacement} from the set $\{0/n, 1/n\ldots, (n-1)/n\}$. Since sampling without replacement has smaller variance, the proposed estimator performs better and in the particular case of this example has zero error.

We will generalize the above result and show a data dependent bound in its error in Theorem~\ref{thm:onedimonebitzq}. For completeness, the full algorithm when each input belongs to the range $[l, r)$ is given in Algorithm~\ref{fig:one_d}. 

\begin{algorithm}[t]
\begin{center}
\begin{minipage}{\textwidth}
\noindent\textbf{Input}: $x_1, x_2,\ldots, x_n, l, r$.\newline
\noindent Generate $\pi$, a random permutation of $\{0, 1, 2,\ldots, n-1\}$. \newline
\noindent For $i = 1 \ \text{to} \ n$:
\begin{enumerate}
\item $y_i = \frac{x_i \zs{- l}}{r-l}$.
\item $U_i = \frac{\pi_i}{n} + \gamma_i$, where $\gamma_i \sim U[0, 1/n)$.
\item $Q_i(x_i) = (r-l) 1_{U_i < y_i}$.
\end{enumerate}
\noindent \textbf{Output}: $\frac{1}{n} \sum^n_{i=1} Q_i(x_i)$.
\end{minipage}
\end{center}
\caption{\textsc{OneDimOneBitCQ}}
\label{fig:one_d}
\end{algorithm}

\begin{restatable}{theorem}{OneDomOneBitCQ}
\label{thm:onedimonebitzq}
If all the inputs lie in the range $[l, r)$, the proposed estimator \textsc{OneDimOneBitCQ} is unbiased and the mean squared error is upper bounded by 
\begin{align*}
   \frac{3}{n} \cdot \sigmamd (r - l) + \frac{12(r - l)^2}{n^2},
\end{align*}
where $\sigmamd$ is defined in~\eqref{eq:sigmamd}.
\end{restatable}
We provide the proof of the theorem in Appendix~\ref{app:onedimonebitzq}. \th{Note that $\sigmamd = 0$ in the toy examples described before.}  We now extend the results to $k$ levels. A standard way of extending one bit quantizers to multiple bits is to divide the interval into multiple fixed length intervals and use stochastic quantization in each interval. While this does provide good worst-case bounds, we cannot get data-dependent bounds in terms of mean deviation. This is because there can be examples in which the samples lie in two different intervals and using stochastic quantization in each interval separately can increase the error. For example, if $k=4$ and we divide the $[0,1]$ into intervals $[0, 1/3]$, $[1/3, 2/3]$, $[2/3, 1]$. If all points are in the set $\{1/3 - \epsilon,   1/3+\epsilon\}$, then they will belong to two different intervals, which yields looser bounds. 

We overcome this, by dividing the input space into randomized intervals. Even though the points may lie in different intervals with randomized intervals, we use the fact that the chance of it happening is small to get bounds in terms of absolute deviation. More formally, let $c_1, c_2, \ldots, c_k$ be $k$ levels such that $c_1$ is uniformly distributed in the interval $\left[ -\frac{1}{k}, 0 \right)$ and $c_i = c_1 + (i-1) \beta$ where $\beta = \frac{k+1}{k(k-1)}$. Observe that with these definitions,
\[
c_k \geq -\frac{1}{k} + (k-1) \cdot \frac{k+1}{k(k-1)} = -\frac{1}{k} + \frac{k+1}{k} = 1.
\]
Let 
\[
c'_i = \max_{c_i < x_i} c_i
\]
If $x \in [0,1]$, we use the following algorithm:
\[
Q_i(x_i) = c'_i + \beta Q'_i\left(\frac{x_i - c'_i}{\beta} \right),
\]
where $Q'_i$ is the two-level quantizer in Algorithm~\ref{fig:one_d}.
The full algorithm when each input belongs to the range $[l, r)$ is given in Algorithm~\ref{fig:one_d_k_bits} and we provide its theoretical guarantee in Theorem~\ref{thm:onedimkbitzq}. We provide the proof of the theorem in Appendix~\ref{app:onedimkbitzq}. 

\begin{algorithm}[t]
\begin{center}
\begin{minipage}{\columnwidth}
\noindent\textbf{Input}: $x_1, x_2,\ldots, x_n, l, r$.\newline
\noindent Let $c_1, c_2, \ldots, c_k$ be $k$ levels such that $c_1$ is uniformly distributed in the interval $\left[ -\frac{1}{k}, 0 \right)$ and $c_i = c_1 + (i-1) \beta$ where $\beta = \frac{k+1}{k(k-1)}$. \newline
\noindent Generate $\pi$, a random permutation of $\{0, 1, 2,\ldots, n-1\}$. \newline
\noindent For $i = 1 \ \text{to} \ n$:
\begin{enumerate}
\item $y_i = \frac{x_i \zs{- l}}{(r-l) \beta}$.
\item $c'_i = \max_{c_i < y_i} c_i$.
\item $U_i = \frac{\pi_i}{n} + \gamma_i$, where $\gamma_i \sim U[0, 1/n)$.
\item $Q_i(x_i) = (r-l) \cdot \left(c'_i + \beta 1_{U_i < y_i} \right)$.
\end{enumerate}
\noindent \textbf{Output}: $\frac{1}{n} \sum^n_{i=1} Q_i(x_i)$.
\end{minipage}
\end{center}
\caption{\textsc{OneDimKLevelsCQ}}
\label{fig:one_d_k_bits}
\end{algorithm}

\begin{restatable}{theorem}{OneDimKLevelsCQ}
\label{thm:onedimkbitzq}
If all the inputs lie in the range $[l, r)$, $k \geq 3$, the proposed estimator \textsc{OneDimKLevelsCQ} is unbiased and the mean squared error is upper bounded by
\begin{align*}
  \frac{12}{n} \cdot \min \left( \frac{\sigmamd \cdot (r - l)}{k}, \frac{(r-l)^2}{k^2} \right)
  + \frac{48(r-l)^2}{n^2k^2},
\end{align*}
where $\sigmamd$ is defined in~\eqref{eq:sigmamd}.
\end{restatable}

\section{Extensions to High Dimensions}

To extend the algorithm to high-dimensional data, we can quantize each coordinate independently using the above quantization scheme. However such an approach is suboptimal. 
In this section, we show that the two approaches used in \citet{suresh2017distributed} namely variable length coding and random rotation can be used here.

\subsection{Variable length coding}

One natural way to extend the above algorithm to high dimensions is to use \textsc{OneDimKLevelsSC} on each coordinate using $k$ bits. \citet{suresh2017distributed, alistarh2017qsgd} observed that while each coordinate is quantized by $k$ bits, instead of using $d \cdot \log_2 k$ bits of communication, one can reduce the communication cost by using variable length codes such as Elias-Gamma codes or Arithmetic coding. We refer to this algorithm as \textsc{EntropyCQ}. We use the same approach and show the following corollary. 
\begin{restatable}{corollary}{EntropyCQ}
\label{cor:entropycq}
If all the inputs lie in the range $\BB^d(R)$, the proposed estimator \textsc{EntropyCQ} is unbiased and the mean squared error is upper bounded by
\begin{align*}
  \frac{c}{n} \cdot \min \left( \frac{\sqrt{d}\sigmadmd R}{k}, \frac{dR^2 }{k^2} \right)
  + \frac{c^2d}{n^2k^2},
\end{align*}
where $\sigmadmd$ is defined in~\eqref{eq:sigmadmd} and $c$ is a constant.
Furthermore, overall the quantizer can be communicated to the server in $\mathcal{O} \left( d \cdot \left(1 + \log_2 \left( \frac{(k-1)^2}{2d} + 1 \right) \right)\right) + k \log_2 \frac{(k+d)e}{k} + \tilde{O}(1) $ bits per client in expectation. 
\end{restatable}
The proof of unbiasedness and variance follows directly from Theorem~\ref{thm:onedimkbitzq} applied to each coordinate. The proof of communication cost is similar to that of  \citep[Theorem 4]{suresh2017distributed} and omitted.
Based on the above corollary, we can set $k = \sqrt{d}$ bits and have a quantizer that uses $\cO(d)$ bits and has error  
\begin{align*}
  \frac{c}{n} \cdot \min \left( \sigmadmd R, R^2 \right)
  + \frac{cR^2}{n^2}.
\end{align*}

\subsection{Random rotation}

Instead of using variable length code, one can use a random rotation matrix to reduce the $L_\infty$ norm of the vectors. We use this approach and show the following result. The proof is given in Appendix~\ref{app:walshcq}. Similar to \citet{suresh2017distributed}, one can use the efficient Walsh-Hadamard rotation which takes   $\mathcal{O}(d \log d)$ time to compute (Algorithm \ref{fig:hadamard_d}). 

\begin{restatable}{corollary}{WalshHadamardCQ}
\label{cor:walshcq}
If all the inputs lie in the range $\BB^d(R)$, the proposed estimator \textsc{WalshHadamardCQ} has bias at most $3R\sqrt{2\log(dn)}/(n^{2}d^{3/2})$ and the mean squared error is upper bounded by
\begin{align*}
     \frac{ c \log (dn)}{n} \cdot \left( \min \left(  \frac{\sigmadmd R}{k}, \frac{R^2}{k^2} \right)
      + \frac{R^2}{nk^2} + \frac{ R^2}{n^3 d^3} \right),
\end{align*}
where $\sigmadmd$ is defined in~\eqref{eq:sigmamd} and $c$ is a constant.
Furthermore, overall the quantizer can be communicated to the server in $d \cdot \log_2 k + \tilde{O}(1) $ bits per client in expectation. 
\end{restatable}
 We note that with communication cost of $\cO(d)$ bits, the bounds with the random rotation are sub-optimal by a logarithmic factor compared to the variable length coding approach. However, it may be the desired approach in practice due to ease of use or computation costs. 
 
 \begin{algorithm}[t]
\begin{center}
\begin{minipage}{\columnwidth}
\noindent\textbf{Input}: $x_1, x_2,\ldots, x_n \in \RR^d, R $.\newline
\noindent For $j \leq d$, let  $\pi^j$, a random permutation of $\{0, 1, 2,\ldots, n-1\}$. \newline
\noindent Let $W = \frac{1}{\sqrt{d}}H D$, where $H$ is a Hadamard matrix and $D$ is a diagonal matrix with independent Rademacher entries\footnotemark.
\newline
\noindent For $i = 1 \ \text{to} \ n$:
\begin{enumerate}
\item $y_i = \frac{\sqrt{d}}{R \sqrt{8\log (dn)}}  \cdot W x_i$.
\item For $j = 1 \ \text{to} \ d$:
\begin{enumerate}
    \item $y'_{i}(j) = \max(-1, \min(y_{i}(j), 1))$.
\end{enumerate}
\end{enumerate}
\noindent For $j = 1 \ \text{to} \ d$:
\begin{enumerate}
\item $z(j) = \textsc{OneDimKLevelsCQ}(y'_{i}(j), \ldots, y'_{n}(j), -1, 1)$.
\end{enumerate}
\noindent \textbf{Output}: $\frac{R \sqrt{8\log (dn)}}{\sqrt{d}}  W^{-1} \cdot z$.
\end{minipage}
\end{center}
\caption{\textsc{WalshHadamardCQ}}
\label{fig:hadamard_d}
\end{algorithm}

\footnotetext{Both matrices are of dimension $d\times d$. When $d$ is not a power of 2, to multiply with Hadamard matrices, we attach 0's at the end of the vectors without affecting asymptotic results.}

\section{Lower Bound}

\input{lower}

\section{Experiments}

\begin{figure}[t]
\centering
\centering
 \begin{subfigure}[t]{.4\linewidth}
    \centering\includegraphics[width=1.0\linewidth]{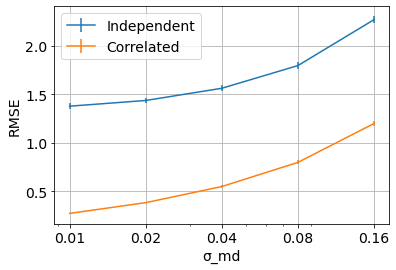}
    \caption{RMSE as a function of $\sigmamd$.}
  \end{subfigure} 
  \begin{subfigure}[t]{.4\linewidth}
    \centering\includegraphics[width=1.0\linewidth]{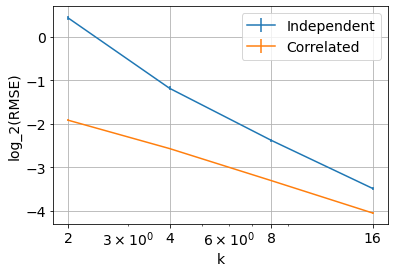}
    \caption{RMSE as a function of $k$.}
  \end{subfigure}
      \begin{subfigure}[t]{.4\linewidth}
    \centering\includegraphics[width=1.0\linewidth]{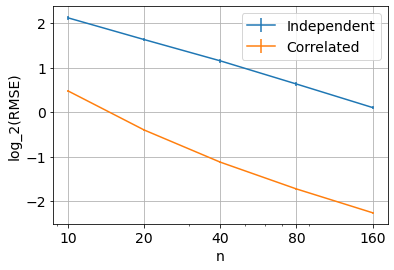}
    \caption{RMSE as a function of $n$.}
  \end{subfigure}
    \begin{subfigure}[t]{.4\linewidth}
    \centering\includegraphics[width=1.0\linewidth]{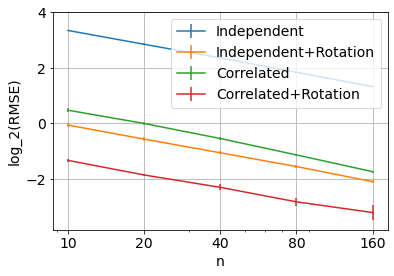}
    \caption{RMSE with random rotation.}
  \end{subfigure}
  \caption{Comparison of compression algorithms on mean estimation task. }
\label{fig:toy_examples}
\end{figure}

\begin{table}[t]
\centering
\caption{Comparison of mean square error of compression algorithms on distributed mean estimation.}
\label{tab:dme-experiments}
\begin{tabular}{lcc}
Algorithm & Synthetic & MNIST \\
\hline
Independent & $10.28 (0.25)$ & $0.466 (0.014)$\\
Independent + Rotation & $3.29 (0.19)$ & $1.661 (0.126)$\\
TernGrad ($\log_{2}(3)$ bits)& $2.64 (0.01)$& $0.621 (0.006)$ \\
Structured DRIVE & $1.38 (0.01)$& $2.165 (0.006)$ \\
\hline
Correlated & $1.40 (0.05)$ & $\mathbf{0.141 (0.004)}$\\
Correlated + Rotation & $\mathbf{1.01 (0.06)}$& $0.238 (0.012)$ \\
\end{tabular}
\end{table}

\begin{table*}[t]
\centering
\caption{Comparison of compression algorithms on a variety of tasks: distributed mean estimation, distributed $k$-means clustering, distributed power iteration, and federated averaging on the MNIST dataset. For all tasks, we set the number of quantization levels to two (one bit), except TernGrad which uses three quantization levels.}
\label{tab:experiments}
\begin{tabular}{lcccc}
Algorithm & Dist. $k$-means & Dist. Power Iteration & FedAvg \\
 & Objective & Error & Accuracy \% \\
\hline
No quantization  & $39.43 (0.12)$ & $0.008 (0.000)$ & $89.31 (0.01)$ \\
\hline
Independent  & $42.98 (0.23)$ & $0.242 (0.005)$ & $87.38 (0.30)$ \\
Independent + Rotation  & $68.85 (2.81)$ & $0.267 (0.016)$ & $89.11 (0.06)$ \\
TernGrad ($\log_{2}(3)$ bits)  & $42.07 (0.09)$ & $0.100 (0.001)$ & $88.50 (0.03)$ \\
Structured DRIVE  & $46.93 (0.17)$ & $0.370 (0.001)$ & $89.00 (0.03)$ \\
\hline
Correlated  & $\mathbf{39.97 (0.13)}$ & $\mathbf{0.055 (0.002)}$ & $88.39 (0.12)$ \\
Correlated + Rotation  & $42.21 (0.17)$ & $0.059 (0.002)$ & $\mathbf{89.20 (0.04)}$ \\
\end{tabular}
\end{table*}

We demonstrate that the proposed algorithm outperforms existing baselines on several distributed tasks.
Before we conduct a full comparison, we first empirically demonstrate that our correlated quantization algorithm is better than existing independent quantization schemes \citep{suresh2017distributed, alistarh2017qsgd}.
We implement all algorithms and experiments using the open-source JAX \citep{jax2018github} and FedJAX \citep{ro2021fedjax} libraries \footnote{\url{https://github.com/google-research/google-research/tree/master/correlated\_compression}}.  \th{We simulated shared randomness by downstream communication from server to clients \citep{szlendak2021permutation}.}

\textbf{Correlated vs independent stochastic quantization.}
We first compare correlated and independent stochastic quantizations on a simple mean estimation task. Let $x_1, x_2,\ldots, x_n$ be $n$ i.i.d.\ samples over $\RR^{1024}$, where coordinate $i$ is sampled independently according to $\mu(i) + U$, where $\mu(i)$ is a uniform random variable between $[0, 1]$ and is fixed for all clients and $U$ is a independent random variable for each client in the range $[-4 \sigmamd, 4 \sigmamd]$.
Note that this distribution has a mean deviation of $\sigmamd$. 
We first fix the number of clients $n$ to be $100$, $k=2$, and vary $\sigmamd$. 
We then fix $\sigmamd = 0.01$, $n=100$ and vary $k$.
Finally, we fix $\sigmamd = 0.01$, $k=2$ and vary $n$.
The results are given in Figures~\ref{fig:toy_examples} $(a), (b), (c)$ respectively. The experiments are averaged over ten runs for statistical consistency.
Observe that in all the experiments, correlated quantization outperforms independent stochastic quantization.

\textbf{Effect of random rotation.}
We next demonstrate that the correlated quantization benefits from random rotation similar to independent quantization. Let $x_1, x_2,\ldots, x_n$ be $n$ i.i.d.\ samples over $\RR^{1024}$, where coordinate $i$ is independently sampled according to $\mu(i) + U$, where $U$ is a independent random variable for each client in the range $[-4 \sigmamd, 4 \sigmamd]$. We let $\mu = (1.0, -1.0, 0.0,\ldots, 0.0)$, $k=2$, and $\sigmamd= 0.01$. 
We compare the results as a function of $n$ in Figure~\ref{fig:toy_examples} $(d)$. Observe that random rotation improves the performance of both correlated and independent quantization. Furthermore, rotated correlated quantization outperforms the remaining schemes.

In the following experiments, we compare our correlated quantization algorithm with several quantization baselines: \emph{Independent}, \emph{Independent+Rotation} \citep{suresh2017distributed}, \emph{TernGrad} \citep{wen2017terngrad}, and \emph{Structured DRIVE} \citep{vargaftik2021drive}.
Since the focus of the paper is quantization, we only compare lossy quantization schemes and do not evaluate the \emph{lossless} compression schemes such as arithmetic coding or Huffman coding, which can be applied on any quantizer.
We use 2-level quantization (one bit) for all the algorithms, except TernGrad which uses $3$ levels and hence requires $\log_2(3)$ bits per coordinate per client. 

\textbf{Distributed mean estimation.}
We next compare our proposed algorithm to existing baselines on a sparse mean estimation task.
Let $x_1, x_2,\ldots, x_n$ be $n$ i.i.d.\ samples over $\RR^{1024}$, where coordinate $i$ is sampled independently according to $\mu(i) + U$, where $\mu$ is a sparse vector with $1\%$ sparse entries and is fixed for all clients and $U$ is a independent random variable for each client in the range $[-4 \sigmamd, 4 \sigmamd]$.
We also compare quantizers on the distributed mean estimation task for the MNIST ($d=784$) dataset distributed over $100$ clients.
The results for both for $10$ repeated trials are in Table~\ref{tab:dme-experiments}.
Observe that correlated quantizers perform best. 

\textbf{Distributed k-means clustering.}
In the distributed Lloyd’s (k-means) algorithm, each client has access to a subset of data points
and the goal of a server is to learn $k$-centers by repeatedly interacting with the clients.
We employ quantizers to reduce the uplink communication cost from clients to server and use the MNIST ($d=784$) dataset and set both the number of centers and number of clients to 10.
We split examples evenly amongst clients and use k-means++ to select the initial cluster centers.
The results for $20$ communication rounds for $10$ repeated trials are in Table~\ref{tab:experiments}.
Observe that correlated quantization performs the best.

\textbf{Distributed power iteration.}
In distributed power iteration, each client has access to a subset of data and the goal of the server is to learn the top eigenvector. Similar to the previous distributed k-means clustering, in the quantized setting, we use quantization to reduce the communication cost
from clients to server and use the MNIST ($d=784$) dataset distributed evenly over $100$ clients
The results for $20$ communication rounds for $10$ repeated trials are in Table~\ref{tab:experiments}.
Observe that  
correlated quantizers outperform all baselines.

\textbf{Federated learning.}
We finally evaluate the effectiveness of the proposed algorithm in reducing the uplink communication costs in federated learning \citep{mcmahan2017communication}. We use the image recognition task for the Federated MNIST dataset \citep{caldas2018leaf} provided by TensorFlow Federated~\citep{tff2019}.
This dataset consists of $341K$ training examples with $10$ label classes distributed across $3383$ clients.
We use quantizers to reduce the uplink communication cost from clients to server and train a logistic regression model for $1000$ communication rounds of federated averaging for $5$ repeated trials.
The results are in Table~\ref{tab:experiments}. 
Observe that correlated quantization with rotation achieves the highest test accuracy.
 
\section{Conclusion}
We proposed a new quantizer for distributed mean estimation and showed that the error guarantee depends on the deviation of data points instead of their absolute range.
We further used this result to provide fast convergence rates for distributed optimization under communication constraints.
Experimental results show that our proposed algorithm outperforms existing compression protocols on several tasks.
We also proved the optimality of the proposed approach under mild assumptions in one dimension.
Extending the lower bounds to high dimensions remains an interesting open question.

\newpage
\bibliographystyle{abbrvnat}
\bibliography{references}

\onecolumn
\appendix

\section{Proof of Theorem~\ref{thm:onedimonebitzq}}
\label{app:onedimonebitzq}
\OneDomOneBitCQ*
\begin{proof}
We first show the result when $l=0$ and $r=1$, one can obtain the final result by rescaling the quantizer operation
\[
(l-u) \cdot Q_i \left( \frac{x_i}{u-l} \right).
\]
Hence in the following, without loss of generality, let $l=0, r=1$. 
 We first show that $Q$ is an unbiased equantizer
 \begin{align*}
    \EE[ Q_i(x_i)]   &= \EE[1_{\frac{\pi_i}{n} + \gamma_i < x_i}].
 \end{align*}
 Observe that $\frac{\pi_i}{n} + \gamma_i$ is a uniform random variable between $[0, 1)$. Hence,
  \begin{align*}
    \EE[ Q_i(x_i)]   &= \EE[1_{\frac{\pi_i}{n} + \gamma_i < x_i}] 
    = \text{Pr}\left(\frac{\pi_i}{n} + \gamma_i < x_i\right) 
    = x_i. 
 \end{align*}
We now bound its variance. To this end let 
\[
y_i = \frac{\lfloor n x_i \rfloor}{n}.
\]
We can rewrite the difference between estimate and the true sum as
\begin{align*}
   \sum^n_{i=1} x_i -  \sum^n_{i=1} Q_i(x_i)
 &   =  \sum^n_{i=1} x_i -  \sum^n_{i=1} y_i 
     +    \sum^n_{i=1} y_i - \sum^n_{i=1} Q_i(y_i) + 
 \sum^n_{i=1} Q_i(y_i) -   \sum^n_{i=1} Q_i(x_i).
\end{align*}
Since $(a+b+c)^2 \leq 3a^2 + 3b^2 + 3c^2$,
\begin{align*}
    \EE \left[\left( \sum^n_{i=1} x_i -  \sum^n_{i=1} Q_i(x_i) \right)^2 \right]
     =& 3\EE \left[\left( \sum^n_{i=1} x_i -  \sum^n_{i=1} y_i \right)^2 \right] \\
     &+  3\EE \left[\left( \sum^n_{i=1} y_i -  \sum^n_{i=1} Q_i(y_i) \right)^2 \right]
    +  3\EE \left[\left( \sum^n_{i=1}  Q_i(y_i) -  \sum^n_{i=1} Q_i(x_i) \right)^2 \right].
\end{align*}
We now bound each of the three terms in the above summation.

\textbf{First term}: Observe that for all $i$
\[
|x_i - y_i| \leq \frac{1}{n}.
\]
Hence,
\[
\left\lvert \sum^n_{i=1}x_i -  \sum^n_{i=1} y_i \right\rvert
\leq \sum^n_{i=1} |x_i - y_i| \leq n \cdot \frac{1}{n} = 1.
\]
Therefore,
\[
\EE \left[\left( \sum^n_{i=1} x_i -  \sum^n_{i=1} y_i \right)^2 \right]  \leq 1.
\]
\textbf{Second term}: To bound the second term,
\begin{align*}
   \EE \left[\left( \sum^n_{i=1} y_i -  \sum^n_{i=1} Q_i(y_i) \right)^2 \right] 
    & =  \sum^n_{i=1}  \EE \left[\left(  y_i -  Q_i(y_i) \right)^2 \right] 
     + \sum^n_{i=1} \sum_{j \neq i} \EE \left[\left(  y_i -  Q_i(y_i) \right) \left(  y_j -  Q_j(y_j) \right)\right] \\
    & = \sum^n_{i=1} y_i (1-y_i)
    + \sum^n_{i=1} \sum_{j \neq i} \EE \left[\left(  y_i -  Q_i(y_i) \right) \left(  y_j -  Q_j(y_j) \right)\right] \\
     & = \sum^n_{i=1} y_i (1-y_i)
    + \sum^n_{i=1} \sum_{j \neq i} \left( \EE \left[ Q_i(y_i)Q_j(y_j) \right] -  y_i y_j\right), 
\end{align*}
where the second equality uses the fact that $1_{\frac{\pi_i}{n} + \gamma_i < y_i}$ is a Bernoulli random variable with parameter $y_i$. We now bound 
$ \EE \left[ Q_i(y_i)Q_j(y_j) \right] -  y_i y_j$ for $i \neq j$. observe that
\[
Q_i(y_i) = 1_{\frac{\pi_i}{n} + \gamma_i < y_i}.
\]
However, since $y_i$ is an integral multiple of $1/n$, and $\gamma_i \in [0, 1/n)$,  $\frac{\pi_i}{n} + \gamma_i < y_i$ if and only if $\frac{\pi_i}{n} < y_i$. Hence,
\[
Q_i(y_i) = 1_{\frac{\pi_i}{n} < y_i}.
\]
Hence,
\begin{align*}
   \EE \left[ Q_i(y_i)Q_j(y_j) \right] 
    = \EE \left[ 1_{\frac{\pi_i}{n} < y_i}  1_{\frac{\pi_j}{n} < y_j} \right] 
    = \text{Pr} \left(\pi_i < n y_i \cap \pi_j < n y_j \right).
\end{align*}
Since $\pi$ is a random permutation,
\begin{align*}
 \text{Pr} \left(\pi_i < n y_i \cap \pi_j < n y_j \right) 
&= \frac{1}{n(n-1)}\min(ny_i, ny_j) \cdot\left( \max(ny_i, ny_j) - 1\right)   \\ 
&= \frac{n y_i y_j}{(n-1)} - \frac{\min(y_i, y_j)}{(n-1)} \\
&= \frac{n y_i y_j}{(n-1)} - \frac{y_i + y_j}{2(n-1)} + \frac{|y_i - y_j|}{2(n-1)}.
\end{align*}
Hence,
\begin{align*}
    \sum^n_{i=1} \sum_{j \neq i} \left( \EE \left[ Q_i(y_i)Q_j(y_j) \right] -  y_i y_j\right) 
    & = \sum^n_{i=1} \sum_{j \neq i} \left( \frac{n y_i y_j}{(n-1)} - \frac{y_i + y_j}{2(n-1)} + \frac{|y_i - y_j|}{2(n-1)} - y_iy_j \right) \\
    & = \frac{1}{n-1} \sum^n_{i=1} \sum_{j \neq i} \left(y_i y_j  - \frac{y_i + y_j}{2} + \frac{|y_i - y_j|}{2}  \right) \\
    & = \frac{1}{n-1} \left( \Bigl(\sum^n_{i=1} y_i\Bigr)^2 -\sum^n_{i=1} y^2_i  - (n-1) \sum^n_{i=1} y_i + \sum^n_{i=1} \sum_{j \neq i} \frac{|y_i - y_j|}{2}  \right)  \\
    & = \sum^n_{i=1} \sum_{j \neq i} \frac{|y_i - y_j|}{2(n-1)} +   \frac{1}{n-1} \left( \Bigl(\sum^n_{i=1} y_i\Bigr)^2 -\sum^n_{i=1} y^2_i\right) - \sum_i y_i \\
      & \leq \sum^n_{i=1}\sum_{j \neq i} \frac{|y_i - y_j|}{2(n-1)} +   \frac{1}{n} \left( \Bigl(\sum^n_{i=1} y_i \Bigr)^2\right) - \sum^n_{i=1} y_i \\
\end{align*}
Hence,
\begin{align*}
 \EE \left[\left( \sum^n_{i=1} y_i -  \sum^n_{i=1} Q_i(y_i) \right)^2 \right]  
    & = 
\sum^n_{i=1} y_i (1-y_i)
    + \sum^n_{i=1} \sum_{j \neq i} \left( \EE \left[ Q_i(y_i)Q_j(y_j) \right] -  y_i y_j\right) \\
    & \leq \sum^n_{i=1} y_i (1-y_i) \zs{+} \sum^n_{i=1}\sum_{j \neq i} \frac{|y_i - y_j|}{2(n-1)} +   \frac{1}{n} \left( (\sum_i y_i)^2\right) - \sum_i y_i \\
    & = \sum^n_{i=1}\sum_{j \neq i} \frac{|y_i - y_j|}{2(n-1)} + \frac{1}{n} \left( \Bigl(\sum^n_{i=1} y_i\Bigr)^2\right)  - \sum^n_{i=1} y^2_i \\
    & \leq  \sum^n_{i=1}\sum_{j \neq i} \frac{|y_i - y_j|}{2(n-1)} \\
      & \leq  \sum^n_{i=1}\sum_{j \neq i} \frac{|x_i - x_j|}{2(n-1)} + \sum^n_{i=1}\sum_{j \neq i} \frac{1}{2n(n-1)} \\  & \leq  \sum^n_{i=1}\sum_{j \neq i} \frac{|x_i - x_j|}{2(n-1)} +1 .
\end{align*}
\textbf{Third term}: Observe that 
\[
Q_i(y_i) - Q_i(x_i) =  1_{\frac{\pi_i}{n} + \gamma_i < y_i} - 1_{\frac{\pi_i}{n} + \gamma_i < x_i}.
\]
Since $x_i \geq y_i$, $1_{\frac{\pi_i}{n} + \gamma_i < y_i}$ implies $1_{\frac{\pi_i}{n} + \gamma_i < x_i}$. Hence $Q_i(y_i) - Q_i(x_i)$ takes at most two values $0, -1$. Therefore,
\[
Q_i(y_i) - Q_i(x_i)  = -1_{y_i \leq \frac{\pi_i}{n} + \gamma_i <x_i}.
\]
Therefore,
\begin{align*}
&    \EE \left[\left( \sum^n_{i=1}  Q_i(y_i) -  \sum^n_{i=1} Q_i(x_i) \right)^2 \right] \\
    & =\sum^n_{i=1}  \EE \left[\left(  Q_i(y_i) -   Q_i(x_i) \right)^2 \right]  + \sum^n_{i=1} \sum_{j \neq i}   \EE \left[\left( Q_i(y_i) - \zs{Q_i(x_i) } \right) \left( Q_j(y_j) -  Q_j(x_j) \right) \right] \\
    & = \sum^n_{i=1}  \EE \left[1_{y_i \leq \frac{\pi_i}{n} + \gamma_i <x_i} \right]   + \sum^n_{i=1} \sum_{j \neq i}   \EE \left[1_{y_i \leq \frac{\pi_i}{n} + \gamma_i <x_i} 1_{y_j \leq \frac{\pi_j}{n} + \gamma_j <x_j} \right] \\
       & = \sum^n_{i=1}  \EE \left[1_{y_i \leq \frac{\pi_i}{n} + \gamma_i <y_i+1} \right]   + \sum^n_{i=1} \sum_{j \neq i}   \EE \left[1_{y_i \leq \frac{\pi_i}{n} + \gamma_i <y_i + 1} 1_{y_j \leq \frac{\pi_j}{n} + \gamma_j <y_j + 1} \right] \\
        & = \sum^n_{i=1}  \EE \left[1_{y_i \leq \frac{\pi_i}{n} <y_i+1} \right]   + \sum^n_{i=1} \sum_{j \neq i}   \EE \left[1_{y_i \leq \frac{\pi_i}{n}  <y_i + 1} 1_{y_j \leq \frac{\pi_j}{n}  <y_j + 1} \right] \\
        & \leq \sum^n_{i=1} \frac{1}{n} + \sum^n_{i=1} \sum_{j \neq i} \frac{1}{n(n-1)} \\
        & = 2.
\end{align*}
Combining the analysis of the above three terms we get
\[
\EE \left[\left( \sum^n_{i=1} x_i -  \sum^n_{i=1} Q_i(x_i) \right)^2 \right]
 \leq \sum^n_{i=1}\sum_{j \neq i} \frac{3|x_i - x_j|}{2(n-1)} + 12.
\]
Hence,
\begin{align*}
    \cE(Q^n, x^n) 
    & \leq \sum^n_{i=1}\sum_{j \neq i} \frac{3|x_i - x_j|}{2n^2(n-1)} + \frac{12}{n^2} \\
    & \leq \frac{3}{2n} \cdot \frac{1}{n(n-1)}\sum^n_{i=1}\sum_{j \neq i} |x_i - x_j| + \frac{12}{n^2} .
\end{align*}
The lemma follows by observing that 
\[
\frac{1}{n(n-1)}\sum^n_{i=1}\sum_{j \neq i} |x_i - x_j| \leq 2\sigmamd.
\]
\end{proof}

\section{Proof of Theorem~\ref{thm:onedimkbitzq}}
\label{app:onedimkbitzq}

\OneDimKLevelsCQ*
\begin{proof}
Similar to the proof of Theorem~\ref{thm:onedimonebitzq}, without loss of generality, we let $l=0, r=1$. 
One can obtain the final result by rescaling the quantizer via
\[
(l-u) \cdot Q_i \left( \frac{x_i}{u-l} \right).
\]
The proof of unbiasedness is similar to the proof of Theorem~\ref{thm:onedimonebitzq} and is omitted.  We now bound its variance.
Let $z_i = \frac{x_i - c'_i}{\beta} $. We wish to bound
\[
\sum^n_{i=1} x_i -  \sum^n_{i=1} Q_i(x_i)
= \beta \left( \sum^n_{i=1}  z_i -  \sum^n_{i=1}  Q_i \left( z_i\right) \right).
\] 
Similar to the proof of Theorem~\ref{thm:onedimonebitzq}, it can be shown that
\[
\EE \left[\left( \sum^n_{i=1} z_i -  \sum^n_{i=1} Q_i(z_i) \right)^2 \vert c' \right]
 \leq \sum^n_{i=1}\sum_{j \neq i} \frac{3|z_i - z_j|}{2(n-1)} + 12.
\]
Hence,
\begin{align*}
\EE \left[\left( \sum^n_{i=1} x_i -  \sum^n_{i=1} Q_i(x_i) \right)^2 \vert c' \right]
  & \leq \beta^2 \sum^n_{i=1}\sum_{j \neq i} \frac{3|z_i - z_j|}{2(n-1)} + 12 \beta^2.
\end{align*}
If $|x_i - x_j| \geq \beta$, then 
\[
 |z_i - z_j| \leq 1.
\]
If $|x_i - x_j| \leq \beta$,
then
\begin{align*}
    \EE[ |z_i - z_j|]
    & = \text{Pr}(c'_i = c'_j) \EE[ |z_i - z_j|1_{c'_i = c'_j}] +  \text{Pr}(c'_i \neq c'_j)  \EE[ |z_i - z_j|1_{c'_i \neq c'_j}]\\
     & \leq  \text{Pr}(c'_i = c'_j) \EE[ |z_i - z_j|1_{c'_i = c'_j}] +  \text{Pr}(c'_i \neq c'_j) \\
     & \leq   \text{Pr}(c'_i = c'_j) \frac{|x_i - x_j|}{\beta} +  \text{Pr}(c'_i \neq c'_j) \\
         & \leq  \frac{|x_i - x_j|}{\beta} +  \text{Pr}(c'_i \neq c'_j) \\
      & \leq  \frac{|x_i - x_j|}{\beta} +  \frac{|x_i - x_j|}{\beta}\ = 2\frac{|x_i - x_j|}{\beta} .
\end{align*}
Hence,
\[
   \EE[ |z_i - z_j|]  \leq 2 \left( \min\left(\frac{|x_i - x_j|}{\beta}, 1\right) \right).
\]
Therefore,
\begin{align*}
 \EE \left[\left( \sum^n_{i=1} x_i -  \sum^n_{i=1} Q_i(x_i) \right)^2 \right]
  & \leq \sum^n_{i=1}\sum_{j \neq i}\beta  \frac{3\min \left(|x_i - x_j|, \beta\right)}{(n-1)} + 12 \beta^2 \\
  & \leq 3\min\left( \sum^n_{i=1}\sum_{j \neq i} \frac{\beta |x_i - x_j|}{(n-1)} , n \beta^2 \right) + 12 \beta^2 \\
    & \leq 3\min\left( 2n\beta\sigmamd , n \beta^2 \right) + 12 \beta^2.
\end{align*}
\begin{align*}
    \cE(Q^n, x^n) 
    & \leq \frac{3}{n} \cdot \min \left(  2 \beta \sigmamd , \beta^2 \right) + \frac{12\beta^2}{n^2} .
\end{align*}
Combining with the fact that $\beta \leq 2/k$ yields the result.
\end{proof}

\section{Proof of Corollary~\ref{cor:walshcq}}
\label{app:walshcq}

\WalshHadamardCQ*
\begin{proof}
The communication cost follows by observing that we are communication $k$ level of quantization for each coordinate. 
We first bound the bias. Observe that
\begin{align*}
    \| \mean - \EE[\hat{x}]\|_2^{2} 
    & \leq \frac{1}{n}  \sum^n_{i=1} \left \lVert x_i - \frac{R \sqrt{8\log (dn)}}{\sqrt{d}} \EE[W^{-1} \cdot z_i] \right \rVert_2^{2}   \\
    &\zs{ = \frac{1}{n}  \sum^n_{i=1} \left \lVert  \EE\left[W^{-1} \cdot \left( W \cdot x_i  - \frac{R \sqrt{8\log (dn)}}{\sqrt{d}} z_i \right)\right]\right \rVert_2^{2}  } \\
    &\zs{ \le  \frac{1}{n}  \sum^n_{i=1} \EE\left[\left \lVert  W^{-1} \cdot \left( W \cdot x_i  - \frac{R \sqrt{8\log (dn)}}{\sqrt{d}} z_i \right)\right \rVert_2^{2}  \right]} \\
    &\zs{ \le  \frac{1}{n}  \sum^n_{i=1} \EE\left[\left \lVert W \cdot x_i  - \frac{R \sqrt{8\log (dn)}}{\sqrt{d}} z_i \right \rVert_2^{2}  \right]} \\
    & \le  \frac{1}{n}  \sum^n_{i=1} 18 R^2 \log (dn)  \Pr\left(\frac{\sqrt{d}}{R \sqrt{8\log (dn)}}W \cdot x_i \neq z_i\right),
\end{align*}
where the last inequality follows from that when $\frac{\sqrt{d}}{R \sqrt{8\log (dn)}} W\cdot x_i \neq z_i$, the error is at most 
\[
    \left\lVert W\cdot x_i -  \frac{R \sqrt{8\log (dn)}}{\sqrt{d}}  z_i \right \rVert_2^2 \le  18 R^2 \log (dn).
\]
Next we bound  $\Pr\left(\frac{\sqrt{d}}{R \sqrt{8\log (dn)}}W \cdot x_i \neq z_i\right) $ for each $x_i$. For any $x$ with $\lVert x\rVert \le R$,

\begin{align*}
    \Pr\left(\frac{\sqrt{d}}{R \sqrt{8\log (dn)}}W \cdot x \neq z\right) & \le \sum_{j = 1}^d \Pr\left(\frac{\sqrt{d}}{R \sqrt{8\log (dn)}}[W x](j) \neq z(j)\right)\\
    & \le \sum_{j = 1}^d \Pr\left([Wx](j) \geq \frac{ \sqrt{8\log (dn))}}{\sqrt{d}} \|x\|_2 \right) \\
    & \stackrel{(b)}{\leq}d e^{-4 \log(dn)},
\end{align*}
where $(b)$ follows from McDiamid's inequality. Hence,
\begin{align*}
     \| \mean - \EE[\hat{x}]\|_2^2  \leq  \frac{18R^2\log(dn)}{d^3 n^4}.
\end{align*}
To bound the mean squared error, observe that
\begin{align*}
    \EE \| \mean - \hat{x}]\|^2_2 & = \EE \| \mean - \EE[\hat{x}]]\|^2_2 + \EE \| \hat{x}- \EE[\hat{x}]]\|^2_2 \\
    & =  \| \mean - \EE[\hat{x}]]\|^2_2 + \EE \| \hat{x}- \EE[\hat{x}]]\|^2_2 \\
    & \leq \| \mean - \EE[\hat{x}]]\|^2_2 + \EE \| W\hat{x}- \EE[W\hat{x}]]\|^2_2 \\
    & \leq \frac{18R^2\log(dn)}{d^3 n^4}+ \EE \| W\hat{x}- \EE[W\hat{x}]]\|^2_2.
    \end{align*}
For $j \in [d]$, let $y^n(j) = (y_1(j), \ldots, y_n(j))$ denote the scalar dataset which only consists of the $j$th coordinate of $y^n$. Define    $\sigmamd(y^n(j))$ to be the corresponding empirical mean absolute deviation.
     \begin{align*}
  \EE \| W\hat{x}- \EE[W\hat{x}]]\|^2_2   & =  \frac{R^2 (8 \log (dn))}{d} \EE \| z- \EE[z]]\|^2_2 \\
     & \leq  \frac{R^2 (8 \log (dn))}{d} \left( \frac{12}{n} \cdot \sum^d_{j=1}\min \left(  \frac{\EE[\sigmamd(y^{n\prime}(j))]}{k}, \frac{1}{k^2} \right)
  + \frac{48d}{n^2k^2} \right) \\
    & \leq  \frac{R^2 (8 \log (dn))}{d} \left( \frac{12}{n} \cdot \sum^d_{j=1}\min \left(  \frac{2\EE[\sigmamd(y^n(j))]}{k}, \frac{1}{k^2} \right)
  + \frac{48d}{n^2k^2} \right) \\
     & \leq  \frac{R^2 (8 \log (dn))}{d} \left( \frac{12}{n} \cdot \min \left(  \frac{2\sum^d_{j=1}\EE[\sigmamd(y^n(j))]}{k}, \frac{d}{k^2} \right)
  + \frac{48d}{n^2k^2} \right) \\
  & \leq  \frac{R^2 (8 \log (dn))}{d} \left( \frac{12}{n} \cdot \min \left(  \frac{2\sqrt{d}\cdot \EE[\sigmadmd(y^n)]}{k}, \frac{d}{k^2} \right)
  + \frac{48d}{n^2k^2} \right) \\
    & \stackrel{(c)}{\leq}  (8 \log (dn)) \left( \frac{12}{n} \cdot \min \left(  \frac{2R\sigmadmd(x^n)}{k}, \frac{R^2}{k^2} \right)
  + \frac{48R^2}{n^2k^2} \right) \\
    & \leq   \frac{192\log (dn)}{n} \cdot \min \left(  \frac{R\sigmadmd(x^n)}{k}, \frac{R^2}{k^2} \right)
  + \frac{392R^2\log (dn)}{n^2k^2},
 \end{align*}
where $(c)$ uses the fact that $ \sigmadmd(y^n) = \frac{\sqrt{d}}{R\sqrt{8 \log (dn)}}\sigmadmd(x^n)$ since $y^n$ is a scaled rotated version of $x^n$.
Combining the above two results yields the result. 
\end{proof}

\input{lower-proof.tex}

\end{document}

%% file: lower.tex
In this section, we discuss information-theoretic lower bounds on the quantization error. We will focus on the one-dimensional case and show that \cq is optimal in terms of the dependence on $\sigmamd$ and $r-l$ under mild assumptions. For the general $d$-dimensional case, whether the dependence on $\sigmadmd$ and $R$ is tight is an interesting question to explore.

In the one-dimensional case with one bit (or constant bits) per client, we obtain the following lower bound, which shows that the upper bound in Theorem~\ref{thm:onedimonebitzq} is tight up to constant factors in terms of the dependence on $\sigmamd$ when $n = \tilde{\Omega}((r-l)/\sigmamd)$. Note that the condition is mild since when $n < (r-l)/\sigmamd$, the second term in  Theorem~\ref{thm:onedimonebitzq}, $(r-l)^2/n^2 > \sigmamd(r - l)/ n$. 
As shown in Theorem~\ref{thm:lower-klevel}, the $(r-l)^2/n^2 $ dependence can not be improved for \monotone quantizers. Whether it can be improved for general quantizers is an interesting future direction.
\begin{restatable}{theorem}{OneDimOneBitLower}
\label{thm:lower-onebit}
    For any $l, r$ and $\sigmamd < \frac{l-r}{2}$, and any one-bit quantizer $Q^n$, when $n > \frac{8(r-l)}{\sigmamd}\log ((r-l)/\sigmamd)$, there exists a dataset $x^n \in [l, r)^n$ with mean absolute deviation $O(\sigmamd)$, such that 
    \[
        \cE(Q^n, x^n) \ge \frac{\sigmamd(r - l)}{64 n}.
    \]
\end{restatable}
Turning to the $k$-level ($\log_2 k$ bits) case, when $\sigmamd < \frac{r - l}{2k}$, we are able to show that our estimator is optimal up to constant factors for a more restricted class of \monotone quantizers. \monotone quantizers include all quantization schemes under which the preimage of each quantized message, ignoring common randomness, is an interval.  To make the definition formal, we slightly abuse notation to assume each quantizer $Q$ admits another argument $\seed \in \{0,1\}^*$, which incorporates all randomness in the quantizer. Fix $s$, $Q(x, s)$ is a deterministic function of $x$.

\begin{definition} \label{def:monotone}
    A quantizer $Q: [l, r) \times \{0, 1\}^* \rightarrow [k]$ is said to be a \monotone quantizer if $\forall s \in \{0, 1\}^*$, there exists $k$ non-overlapping intervals $I_1, \ldots, I_k$ which partitions $[l, r)$, and $\forall j \in [k]$, $x_1, x_2 \in I_j$,
   $
        Q(x_1, s) = Q(x_2, s).
    $
\end{definition}
The class of \monotone quantizers
includes many common compression algorithms used in distributed optimization such as stochastic quantization and our proposed algorithm. 
For this restricted class of schemes, we get

\begin{restatable}{theorem}{OneDimKBitLower} \label{thm:lower-klevel}
    Given $l < r \in \mathbb{R}, k \in \mathbb{N}$ and $\sigmamd < \frac{r-l}{2k}$, for any \monotone quantizer $Q^n$, there exists a dataset $x^n \in [l, r)^n$ with mean absolute deviation $O(\sigmamd)$,  such that
    \[
        \cE(Q^n, x^n) \ge \frac{\sigmamd(r - l)}{64 n k} +  \frac{(r - l)^2}{128 n^2 k^2} .
    \]
\end{restatable}

\th{Under the condition, $\sigmamd < \frac{r-l}{2k}$, the above theorem shows that \textsc{OneDimKLevelsCQ} is nearly-optimal in the class of \monotone quantizers.} We defer the proof of the theorems to Appendix~\ref{app:lower-proof}.

%% file: lower-proof.tex
\section{Proof of Lower bounds (Theorem~\ref{thm:lower-onebit} and Theorem~\ref{thm:lower-klevel})} \label{app:lower-proof}
    A general $k$-level randomized quantizer $Q^n$ can be viewed as a distribution over $\cQ_{\text{det}}$, which contains all deterministic mappings from $[l, r]^n$ to $[k]^n$. \zs{Let $\cQ_{\text{rand}}$ be all randomized mappings from $[l, r]^n$ to $[k]^n$.} To establish a lower bound for general quantizer, we use Yao's minimax principle to reduce it to a lower bound over $\cQ_{\text{det}}$,
    \begin{align} \label{eqn:minmax}
         \min_{Q^n \zs{\in \cQ_{\text{rand}}}}  \max_{x^n} \cE(Q^n, x^n) \ge \max_{P} \min_{Q^n \in \cQ_{\text{det}}}\EE_{X^n \sim P}   \left[ \cE(Q^n, X^n)\right]  =  \max_{P} \min_{Q^n \in \cQ_{\text{det}}} \EE_{X^n \sim P} \left[
        \left(\est(Q^n(X^n))- \frac1n \sum_{i = 1}^n X_i\right)^2 \right].
    \end{align}
In the rest of the proof, we focus on deterministic $Q_i$'s and bound~\eqref{eqn:minmax}. Without loss of generality we assume $l = 0$.  We first prove the following general lemma, which will be useful to prove both lower bounds.
\begin{lemma} \label{lem:mse} Suppose under distribution $P$ over $[0, r]^n$,  $X_i$'s are sampled i.i.d from $p$ over $[0, r]$, for any deterministic quantizer $Q^n$, we have
    \begin{align*}
     \EE_{X^n \sim P} \left[
        \left(\est(Q^n(X^n))- \frac{1}{n} \sum^n_{i=1} X_i \right)^2 \right]
         \ge \frac{1}{n^2} \sum_{i = 1}^n \EE_{X_i \sim p} \left[ \left(X_i - \EE \left[ X_i \mid Q_i(X_i)\right]\right)^2 \right].
\end{align*}
\end{lemma}

\begin{proof}
    Since conditional expectation minimizes the mean squared error over all functions of $\est(Q^n(X^n))$, we have
\begin{align*}
    \EE_{X^n \sim P} \left[
        \left(\est(Q^n(X^n))- \frac{1}{n} \sum^n_{i=1} X_i \right)^2 \right] \ge \EE_{X^n \sim P} \left[
        \left(\EE \left[\frac{1}{n} \sum_{i = 1}^n X_i \mid Q^n(X^n)\right] - \frac{1}{n} \sum^n_{i=1} X_i\right)^2 \right].
\end{align*}
 
Note that since $X_i$'s are independent and $Q^n$ is deterministic, we have all $Q_i(X_i)$ are independent. Moreover, this implies $Q_j(X_j)$ is independent of $X_i$ for $i \neq j$. Hence
\begin{align*}
    \EE \left[\frac{1}{n} \sum_{i = 1}^n X_i \mid Q^n(X^n)\right] = \frac{1}{n} \sum_{i = 1}^n \EE \left[ X_i \mid Q_i(X_i)\right].
\end{align*}
Combining these, we get
\begin{align*}
    \EE_{X^n \sim P} \left[
        \left(\est(Q^n(X^n))- \frac{1}{n} \sum^n_{i=1} X_i\right)^2 \right] \ge \frac{1}{n^2} \sum_{i = 1}^n \EE_{X_i \sim p} \left[ \left(X_i - \EE \left[ X_i \mid Q_i(X_i)\right]\right)^2 \right].
\end{align*}
\end{proof}
Next we use Lemma~\ref{lem:mse} to prove Theorems~\ref{thm:lower-onebit} and~\ref{thm:lower-klevel}.

\OneDimOneBitLower*
\begin{proof}
Consider the following distribution $p$ over $[0,r]$ where 
\[
    p\left(x\right) = \begin{cases}
    \frac{\sigmamd}{2r} & \text{ if } x = 0, \\
    \frac{\sigmamd}{2r} & \text{ if } x = r,\\
    1 - \frac{\sigmamd}{r} & \text{ if } x = \frac{r}{2}, \\
    0 & \text{ otherwise.}
    \end{cases}
\]
We first show that if $P$ is the distribution where $X^n$ is sampled i.i.d from $p$,
\[
    \EE_{X^n \sim P} \left[
        \left(\est(Q^n(X^n))- \frac1n \sum_{i = 1}^n X_i\right)^2 \right] \ge \frac{\sigmamd \cdot r}{32n}.
\]
Applying Lemma~\ref{lem:mse}, it is enough to show that for any $Q_i$, we have
\[
    \EE_{X_i \sim p} \left[ \left(X_i - \EE \left[ X_i \mid Q_i(X_i)\right]\right)^2 \right] \ge  \frac{\sigmamd \cdot r}{32}.
\]

Since $Q_i$ only outputs 1 bit, there must exist two numbers in $\{0, r/2, r\}$ such that they are mapped to the same output. \th{If $\{0, r/2\}$ are mapped to the same output, we have  $\EE \left[ X_i \mid Q_i(X_i) = Q_i(0)\right] = \left( \frac{1 - \frac{\sigmamd}{r}}{1 - \frac{\sigmamd}{2r}} \right) \frac{r}{2} \geq \frac{r}{4}$. Hence,
\begin{align*}
     \EE_{X_i \sim p} \left[ \left(X_i - \EE \left[ X_i \mid Q_i(X_i)\right]\right)^2 \right] 
     &   \ge \text{Pr}\left(X_i = 0\right)  \cdot  \left( \EE \left[ X_i \mid Q_i(X_i) = Q_i(0)\right]\right)^2   \\
 & \ge  \frac{\sigmamd}{2r} \left(\frac{r}{4} \right)^2 = \frac{\sigmamd \cdot r}{32}.
\end{align*}
}
The same result holds when $\{r/2, r\}$ are mapped to the same output. When  $\{0, r\}$ are mapped to the same output, we have
\begin{align*}
     \EE_{X_i \sim p} \left[ \left(X_i - \EE \left[ X_i \mid Q_i(X_i)\right]\right)^2 \right] 
 &   \ge \text{Pr}\left(X_i = 0 \text{ or } r\right) \cdot   \EE_{X_i\mid Q_i(X_i) = Q_i(0) } 
    \left[ \left(X_i - \EE \left[ X_i \mid Q_i(X_i) = Q_i(0)\right]\right)^2  \right] \\
 & \ge \frac{\sigmamd}{r} \times \left(\frac{r}{2} \right)^2 = \frac{\sigmamd \cdot r}{4}.
\end{align*}
The last step is to check the mean absolute deviation. The proof is not complete yet since not all samples $X^n$ from the distribution $P$ has mean deviation $O(\sigmamd)$. Next we construct a distribution $P'$ where all $X^n \sim P'$ has mean deviation $O(\sigmamd)$ while the quantity in \eqref{eqn:minmax} is still lower bounded by $\Omega(\sigmamd \cdot r)$. We start by noting that the mean deviation can be upper bounded as
\begin{align*} \label{eqn:avg_diff}
    \frac{1}{n}\sum_{i = 1}^n |X_i - \frac{1}{n}\sum_{i = 1}^n X_i| \le \frac{1}{n^2}\sum_{i, j \in[n]} |X_i - X_j| = \frac{1}{n^2} \frac{r}{2}\left( N_0 N_{\frac{r}2} + N_r N_{\frac{r}2} + 2N_0 N_r\right),
\end{align*}
where $N_0, N_{\frac{r}{2}}, N_r$ denote the number of appearances of $0, \frac{r}{2}$ and $r$ respectively. When $N_{0} + N_r < \frac{4 n \sigmamd}{r}$, we have
\[
    \frac{1}{n^2} \frac{r}{2}\left( N_0 N_{\frac{r}2} + N_r N_{\frac{r}2} + 2N_0 N_r\right) \le  \frac{1}{n^2} \frac{r}{2} \left(n \cdot \frac{4n \sigmamd}{r} + \left(\frac{2n \sigmamd}{r} \right)^2 \right) \le 4\sigmamd.
\]
Since $N_{0} + N_r$ is a binomial random variable with $n$ trials and success probability $\frac{\sigmamd}{r}$, by Chernoff bound, we have
\[
    \text{Pr} \left( N_{0} + N_r \ge \frac{4 n \sigmamd}{r} \right) \le \exp \left( - \frac{2n\sigmamd}{r}\right) \le \left(\frac{\sigmamd}{r}\right)^{16}.
\]
Let $P'$ be the distribution of $X^n$ conditioned on the event that $N_{0} + N_r < \frac{4 n \sigmamd}{r}$, we have
\begin{align*}
    \EE_{X^n \sim P} \left[
        \left(\est(Q^n(X^n))- \frac1n \sum_{i = 1}^n X_i\right)^2 \right] & \le \text{Pr} \left( N_{0} + N_r < \frac{4 n \sigmamd}{r} \right)  \EE_{X^n \sim P'} \left[
        \left(\est(Q^n(X^n))- \frac1n \sum_{i = 1}^n X_i\right)^2 \right] \\
        & \;\;\;\;\;\; \;\;\;\;\;\;\;\;\;\;\;\; + r^2 \text{Pr} \left( N_{0} + N_r \ge \frac{4 n \sigmamd}{r} \right),
\end{align*}
since when $N_{0} + N_r \ge \frac{4 n \sigmamd}{r}$, the error is bounded by at most $r^2$ since the \th{the optimal} mean lies in $[0,r]$. Hence we have
\begin{align*}
    \EE_{X^n \sim P'} \left[
        \left(\est(Q^n(X^n))- \frac1n \sum_{i = 1}^n X_i\right)^2 \right] & \ge \EE_{X^n \sim P} \left[
        \left(\est(Q^n(X^n))- \frac1n \sum_{i = 1}^n X_i\right)^2 \right]  - r^2 \text{Pr} \left( N_{0} + N_r \ge \frac{4 n \sigmamd}{r} \right) \\
        & \ge \frac{\sigmamd \cdot r}{64n}.
\end{align*}
This completes the proof using~\eqref{eqn:minmax}.

\end{proof}

\OneDimKBitLower*
\begin{proof}
We show that 
\[
 \cE(Q^n, x^n) \ge \max \left( \frac{\sigmamd(r - l)}{32 n k} +  \frac{(r - l)^2}{64 n^2 k^2} \right).
\]
The lemma follows by the fact that maximum is lower bounded by the average. Note that here it would be enough to prove the two terms in the above lower bound separately.
Assuming $l = 0$, for the $\frac{\sigmamd \cdot r}{32 n k}$ term, we assume $n =\Omega( \frac{r}{k \sigmamd})$ since otherwise the second term dominates. We consider $2 k$ distributions $p_1, p_2, \ldots, p_{2k}$ where for $j \in \{1, \ldots, 2k\}$,
\[
    p_j\left(x\right) = \begin{cases}
    \frac{k \sigmamd}{r} & \text{ if } x = (j-1)  \cdot \frac{r}{2k}, \\
    1 - \frac{k \sigmamd}{r} & \text{ if } x = j \cdot\frac{r}{2k}, \\
    0 & \text{ otherwise.}
    \end{cases}
\]
Then we define $\bar{P}$ over $[0, r]^n$ as a uniform mixture of $b$ distributions,
\[
    \bar{P} = \frac{1}{2k} \sum_{j = 1}^{2k} P_j.
\]
where under $P_j$, $X_i$'s are $n$ i.i.d samples from $p_j$. Then we have for any deterministic quantization scheme $Q^n$,
\[
 \EE_{X^n \sim \bar{P}} \left[
        \left(\est(Q^n(X^n))- \frac1n\sum_{i = 1}^n X_i \right)^2 \right] = \frac{1}{2k} \sum_{j = 1}^{2k} \EE_{X^n \sim P_{j}} \left[
        \left(\est(Q^n(X^n))- \frac1n\sum_{i = 1}^n X_i \right)^2 \right] 
\]
Note that each $P_j$ is a product distribution, hence applying Lemma~\ref{lem:mse}, we have
\begin{align*}
    \EE_{X^n \sim \bar{P}} \left[
        \left(\est(Q^n(X^n))- \frac1n\sum_{i = 1}^n X_i \right)^2 \right] & \ge \frac{1}{2k} \sum_{j = 1}^{2k} \frac{1}{n^2} \sum_{i = 1}^n \EE_{X_i \sim p_j} \left[ \left(X_i - \EE \left[ X_i \mid Q_i(X_i)\right]\right)^2 \right] \\
        & = \frac{1}{n^2} \sum_{i = 1}^n   \left( \frac{1}{2k} \sum_{j = 1}^{2k}  \EE_{X_i \sim p_j} \left[ \left(X_i - \EE \left[ X_i \mid Q_i(X_i)\right]\right)^2 \right] \right).
\end{align*}
Hence it is enough to show that for any \monotone quantization scheme, we have $\forall i \in [n]$,
\[
    \frac{1}{2k} \sum_{j = 1}^{2k}  \EE_{X_i \sim p_j} \left[ \left(X_i - \EE \left[ X_i \mid Q_i(X_i)\right]\right)^2 \right] \ge \frac{\sigmamd \cdot r}{32 k}.
\]
A \monotone quantizer $Q_i$ only partitions the interval into $k$ intervals. Hence among the $2k - 1$ pairs of consecutive points $(j \cdot \frac{r}{2k}, (j + 1) \cdot \frac{r}{2k}), j \in \{1, \ldots, 2k-1\}$,  there exists a set $S$ of indices with $ |S| \ge k - 1$ that satisfies $\forall j \in S, Q_i(j \cdot \frac{r}{2k}) = Q_i((j+1) \cdot \frac{r}{2k})$. For these $j$'s, $p_j$ is supported on $\{ j \cdot \frac{r}{2k},(j+1) \cdot \frac{r}{2k} \}$. Hence $Q_i(X_i)$ doesn't bring any information about $X_i$. Hence we have
\[
    \frac{1}{2k} \sum_{j = 1}^{2k}  \EE_{X_i \sim p_j} \left[ \left(X_i - \EE \left[ X_i \mid Q_i(X_i)\right]\right)^2 \right]  \ge \frac{1}{2k} \sum_{j \in S} \text{Var}_{X_i \sim p_j} (X_i) \ge \frac{\sigmamd \cdot r}{32 k}.
\]
Similar to the proof of Theorem~\ref{thm:lower-onebit}, we handle the requirement on mean absolute deviation, which we omit here for brevity. 

For the $\frac{r^2}{64 n^2 k^2}$ term, we consider the following $2 nk$ datasets where for the $j$th dataset, we have $\forall i \in [n]$,
\[ 
    x_i = v_j = (j - 1)\cdot \frac{r}{2 n k}.
\]
We abuse notation to denote the $j$th dataset as $(v_j)^n$. Consider distribution $P$ which is uniform over all $2nk$ datasets. Then we have for any deterministic $Q^n$,
\[
    \EE_{X^n \sim P} \left[
        \left(\est(Q^n(X^n))- \frac1n \sum_{i = 1}^n X_i\right)^2 \right] = \frac{1}{2nk} \sum_{j= 1}^{2nk} \left(\est(Q^n((v_j)^n))- v_j \right)^2.
\]  
Note that there are $n$ users, each with a \monotone quantizer.  Similar to the previous construction, among the $2nk$ possible values, there must exists at least $nk - 1$ disjoint pairs $\{j_1, j_2\}$ such that 
\[
    Q^n((v_{j_1})^n) = Q^n((v_{j_2})^n).
\]
For each of these pairs,  
\[
\left(\est(Q^n((v_{j_1})^n))- v_{j_1} \right)^2 + \left(\est(Q^n((v_{j_2})^n))- v_{j_2} \right)^2 \ge \frac{1}{4} (v_{j_1}  - v_{j_2} )^2 \ge \frac{r^2}{16n^2 k^2}.
\]
Hence we have 
\[
     \frac{1}{2nk} \sum_{j= 1}^{2nk} \left(\est(Q^n((v_j)^n))- v_j \right)^2 \ge \frac{nk - 1}{2nk} \cdot \frac{r^2}{16n^2 k^2} \ge  \frac{r^2}{64n^2 k^2}.
\]
\end{proof}